\documentclass{article}

\usepackage{PRIMEarxiv}

\usepackage[utf8]{inputenc} 
\usepackage[T1]{fontenc}    
\usepackage{hyperref}       
\usepackage{url}            
\usepackage{booktabs}       
\usepackage{amsfonts}       
\usepackage{nicefrac}       
\usepackage{algorithm}
\usepackage{algpseudocode}
\usepackage{microtype}      
\usepackage{amsmath}
\DeclareMathOperator{\argmin}{argmin}
\usepackage{lipsum}
\usepackage{fancyhdr}       
\usepackage{graphicx}       
\graphicspath{{media/}}     
\newtheorem{theorem}{Theorem}
\newtheorem{lemma}{Lemma}
\newtheorem{proof}{Proof}
\newtheorem{assumption}{Assumption}
\usepackage{graphicx}
\usepackage{cite}
\title{Computational Complexity of Sub-linear Convergent Algorithms}

\author{
  Hilal AlQuabeh \\
  Machine Learning Department \\
  MBZUAI \\
  Abu Dhabi, UAE\\
  \texttt{hilal.alquabeh@mbzuai.ac.ae} \\
   \And
  Farha AlBreiki \\
  Machine Learning Department \\
  MBZUAI \\
  Abu Dhabi, UAE\\
  \texttt {farha.albreiki@mbzuai.ac.ae} \\
     \And
  Dilshod Azizov \\
  Natural Language Processing Department \\
  MBZUAI \\
  Abu Dhabi, UAE\\
  \texttt {dilshod.azizov@mbzuai.ac.ae} \\
}

\begin{document}
\maketitle

\begin{abstract}

Optimizing machine learning algorithms that are used to solve the objective function has been of great interest. Several approaches to optimize common algorithms, such as gradient descent and stochastic gradient descent, were explored. One of these approaches is reducing the gradient variance through adaptive sampling to solve large-scale optimization's empirical risk minimization (ERM) problems. In this paper, we will explore how starting with a small sample and then geometrically increasing it and using the solution of the previous sample ERM to compute the new ERM.  This will solve ERM problems with first-order optimization algorithms of sublinear convergence but with lower computational complexity. This paper starts with theoretical proof of the approach, followed by two experiments comparing the gradient descent with the adaptive sampling of the gradient descent and ADAM with adaptive sampling ADAM on three datasets; MNIST, RCV1 and a1a. 

\end{abstract}


\section{Introduction}
The fundamental goal of machine learning algorithms is to identify the conditional distribution given any input and its label. In the training phase, it’s conventional to assume that the underlying classifier or function belongs to a certain class of functions. Therefore presuming that the approximation error is insignificant would be a necessary practice. This practice allows the training to emphasize on what is more practical to reduce the estimation error, which is the major error a classifier develops due to incomplete data training. The estimation error can be further decomposed into optimization and generalization errors, which are greatly complementary.
\par 
Convexity, strong convexity, smoothness, and other features of the objective function (loss function) influence the optimization error. Furthermore, the convergence rate of the optimization problem relies on the algorithm used to solve it. For example, some algorithms have a linear convergence rate, and some have a sublinear or superlinear convergence rate. The computational complexity of an algorithm is a measure of how much computer resources the algorithm utilizes to solve the optimization problem. As a result, computational complexity can be quantified in units of storage, time, dimension, or all three simultaneously.

\par 
A common methodology to quantify the computational complexity of optimization algorithms is by counting entire gradient evaluations required to obtain an optimal solution with a given accuracy $\epsilon$. The Gradient Descent algorithm is the most popular deterministic optimization algorithm with a linear convergence rate assuming $\mu$-strongly convex and $L$-smooth functions and a computational complexity of $\mathcal{O}( \frac{L}{\mu} N \log\frac{1}{\epsilon} )$ for ${N}$ data objective function. On the other hand, the Stochastic Gradient Descent is the most common algorithm that randomly picks a single function every iteration and thus has different computational complexity iteration $\mathcal{O}(\frac{1}{\epsilon} )$. 
When $N$ is large, the preferred methods for solving the resulting optimization or sampling problem usually rely on stochastic estimates of the gradient of $f$. \par


Standard variance reduction techniques used for stochastic optimizations require additional storage or the computation of full gradients.  Another approach for variance reduction is through adaptively increasing the sample size used to compute gradient approximations. 

Some adaptive sampling optimization methods sizes have been studied in \cite{6-samplesize,Raghu-samplesize, daneshmand2016starting, Hashemi-samplesize, mokhtari2017first}. These methods have optimal complexity properties, making them useful for various applications. \cite{Hashemi-samplesize} uses variance-bias ratios to consider a test that is similar to the norm test, which is reinforced by a backup mechanism that ensures a geometric increase in the sample size.  \cite{Raghu-samplesize} establishes terms for global linear convergence by investigating methods that sample the gradient and the Hessian. Other noise reduction methods like SVRG, SAG, and SAGA,  either compute the full gradient at regular intervals or require storage of the component gradients, respectively \cite{Ayoub-samplesize} .

\section{Problem Definition}
The ultimate goal of most machine learning algorithms is to estimate the underlying distribution e.g.: $\mathcal{D}(\mathcal{X},\mathcal{Y})$ where $\mathcal{X}$ is the input feature space, and $\mathcal{Y}$ is the label space, in terms of some hypothesis function $h:\mathcal{X} \rightarrow \mathcal{Y}$, further on we assume that $h$ is determined by a parameter $w$ \cite{shalev2014understanding}. 
Given any set $\mathcal{H}$ (that plays the role of hypotheses space) and domain $\mathbf{Z}:\mathcal{X\times Y}$: let $\ell$ be any function that maps from $\mathcal{H} \times \mathbf{Z}$ to the set of non-negative real numbers e.g. $\ell$ :
	$\mathcal{H} \times \mathbf{Z} \rightarrow \mathbb{R}_+$ . The ability of the proposed hypothesis to estimate the underlying distribution is assessed by such loss functions. 
	The expected loss of a classifier, $h \in \mathcal{H}$, with regard to a probability distribution $\mathcal{D}$ over $Z$ is measured by the risk function \ref{eq:1}.
	
	\begin{equation}
	\label{eq:1}
	L_{\mathcal{D}}(h) := \mathbb{E}_{z\sim \mathcal{D}} (\ell(h,z))
	\end{equation}
    Since this Expected Loss is built on the unknown distribution $\mathcal{D}$, the empirical risk over a given sample of the data $S = (z_1,\dots, z_m) \in Z_m$ is proven to be a good estimator of the expected loss, namely,

\begin{equation}
\label{eq:ERM}
L_s(h) := \frac{1}{m} \sum_{i=1}^{m}\ell(h,z_i)
\end{equation}

\label{sec:Definition}
\subsection{Computational complexity} The computational complexity is used to relate an excess error's upper bound (if one exists) to the available computational resources. Not only the convergence rate, but also the computational resources utilized to accomplish that convergence rate is important in order to have a superior algorithm. If we define a family $\mathcal{H}$ of candidates prediction functions, and let $h_s :=\argmin_{h\in\mathcal{H}}L_s(f)\implies$ ERM solution, $h^*:\argmin_{h}L_{\mathcal{D}}(h)\implies$ True solution (unknown)
 $h_{\mathcal{H}} :=\argmin_{h\in\mathcal{H}}L_{\mathcal{D}}(h)\implies$ Best in class solution (unknown)
We can decompose the true loss (excess loss) as follow:
\[   \mathcal{E}(h_s,h^*) = \mathbb{E}[L(h^*_{\mathcal{H}})-L(h^*)]+\mathbb{E}[L(h_s)-H(h^*_{\mathcal{H}})] = \mathcal{E}_{app} + \mathcal{E}_{est}
\] Where the expectation w.r.t the samples.
\begin{itemize}
	\item The approximation error $ \mathcal{E}_{app} $ measures how closely functions in $\mathcal{H}$
	can approximate the optimal solution $h^*$.
	\item The estimation error $\mathcal{E}_{est}$  measures the effect of minimizing the empirical risk $L_s(f )$ instead of the expected risk $L(f )$.
	\item The estimation error is determined by the number of training
	examples and the capacity of the family of functions.
	\item The estimation error can be bounded using Rademacher
	Complexity to measure the complexity of a family of functions.
	\end{itemize}
	
 Since the empirical risk $L_s (h )$ is already an approximation of the expected risk $L(h )$, it should not be necessary to carry out this
minimization with great accuracy. Assume that our algorithm returns an approximate solution $\hat{h}_s$ such that 
\begin{equation}
\label{eq:rho}
   L_s(\hat{h}_s) \leq L_s(h_s) +\delta_s
    \end{equation}
where $\delta_s$ is a predefined positive tolerance. The new excess error can be decomposed as follows;
\begin{equation}
\mathcal{E}(\hat{h}_s,h^*) = \mathbb{E}[L(h^*_{\mathcal{H}})-L(h^*)]+\mathbb{E}[L(h_s)-L(h^*_{\mathcal{H}})]+\mathbb{E}[L(\hat{h}_s)-L(h_s)]= \mathcal{E}_{app}+\mathcal{E}_{est}+\mathcal{E}_{opt}
\end{equation}
Where the expectation w.r.t the samples. The additional term $\mathcal{E}_{opt}$ is optimization error. It reflects the impact of the approximate optimization on the generalization performance.
\subsection{General Model:}
	The decomposition of the excess error leads to a trade-off minimization taking into account the number samples and allocating computation resources.
\begin{equation}
	\begin{matrix}
		\min_{\mathcal{H},\rho,m} & \mathcal{E}_{app} + \mathcal{E}_{est} + \mathcal{E}_{opt}\\
		s.t. &m  \leq m_{max}\\
		 &T(\mathcal{H},\rho,m)  \leq T_{max}
	\end{matrix} 
	\end{equation}
The variables are the size of the family of functions ${F}$, the optimizaton accuracy $\rho$ within the allotted training time $T_{max}$ , and the number of examples ${n}$ altered by using a subset of all available $n_max$ samples.\\
Typically when the size the class increases, the approximation error decreases, but the estimation error increases and nothing happens to optimization error because it's not related, but the computation time increase. When n increases the estimation error decreases and computation time increases, but no relation to approximation error or optimization error. When $\rho$ increases , the optimization error increases by definition, and computation time decreases.
\subsection{Statistical Error Minimization}In this paper we investigate the statistical error component of the excess error, which just comprises the difference of expected loss in some class and the empirical loss as shown below.
\[ \mathcal{E}_{stat} := \mathcal{E}(\hat{f}_n,f^*_{\mathcal{F}})= \mathbb{E}[E(\hat{f}_n)-E(f^*_{\mathcal{F}})]\]
Further we can add and substract some terms to have:
	\begin{align*}
		 \mathcal{E}_{stat} := \mathcal{E}(\hat{f}_n,f^*_{\mathcal{F}})&= \mathbb{E}[E(\hat{f}_n)-E_n(\hat{f}_n)] +\underbrace{\mathbb{E}[E_n(\hat{f}_n)-E_n(f_n)]}_{\mathbb{E}[E_n(f_n)]=E(f_n)}\\
		 &+\underbrace{\mathbb{E}[E_n({f}_n)-E_n({f}^*_{\mathcal{F}})]}_{\leq0}+\underbrace{\mathbb{E}[E_n({f}^*_{\mathcal{F}})-E({f}^*_{\mathcal{F}})]}_{=0}\\
		 &\leq \underbrace{\mathbb{E}[E(\hat{f}_n)-E_n(\hat{f}_n)]}_{\mathcal{E}_{gen}} +\underbrace{\mathbb{E}[E_n(\hat{f}_n)-E_n(f_n)]}_{\mathcal{E}_{opt}}\\
	\end{align*}
As a result, our primary goal is to minimize statistical error as follows:
	\begin{equation}
	\label{eq:min}
	\begin{matrix}
		\min_{\rho,m} & \mathcal{E}_{gen} + \mathcal{E}_{opt}\\
		s.t. &m  \leq m_{max}\\
		 &T(\rho,m)  \leq T_{max}
	\end{matrix} 
	\end{equation}
We list our assumptions below:
\begin{assumption}[Lipschits Continuity]\label{ass:Lipschitz continuous}
Assume for any $z\in\mathcal{Z}$, the loss function $l(\cdot;z)$ is G-Lipschitz continuous, i.e. $\forall w \in \mathcal{H}$,
\begin{align}
    \nonumber |l(w;z)-l(w';z)|\leq G\left\|w-w'\right\|_2.
\end{align}
\end{assumption}

\begin{assumption}[Convexity]\label{ass:Convexity}
Assume for any $z\in\mathcal{Z}$ the loss function $l(\cdot;z)$ is convex function, i.e. $\forall w \in \mathcal{H}$,
\begin{align}
    \nonumber l(w';z) \geq l(w;z) + \nabla l(w;z)^T (w' - w).
\end{align}
\end{assumption}
\begin{assumption}[L-Smooth]\label{ass:Smooth}
Assume for any $z\in\mathcal{Z}$ the loss gradient function $\nabla l(\cdot;z)$ is is L-Lipschitz continuous, i.e. $\forall w \in \mathcal{H}$,
\begin{equation}
    \| \nabla \ell(w;z) - \ell(w';z)\| \leq L \| w - w' \|
\end{equation}
\end{assumption}
\section{Methodology}
\label{sec:others}
The contribution of this work is mainly deriving the computational complexity of sub-linear convergent algorithm with adaptive sample size training. In order to derive that we start by the generalized  bound that haven been studied well in literature \cite{boucheron2005theory} :
\[ \mathbb{E} \left[\sup_{h\in\mathcal{H}} |L_s(h) - L(h)|   \right] \leq V_n \approx  \mathcal{O}\left(\frac{1}{n^\alpha}\right)
\] 
where $\alpha \in [0 \; 0.5]$ depends on an algorithm used to solve ERM and other factors. The bound is found by \cite{vapnik1999nature} to be $\mathcal{O}(\sqrt{1/n \log{1/n}})\geq\mathcal{O}(\sqrt{1/n }) $, while in other references e.g. \cite{bartlett2006convexity} the bound is improved to $\mathcal{O}(1/n)$ under extra conditions in the regularizer. In any situation, the bound indicates that regardless of the ERM solution's optimization accuracy, there will always be a bound in the order $\mathcal{O}({1}/{n^\alpha})$, thus solving the ERM optimization problem with accuracy $\delta_s$ = 0 would not be beneficial to the final statistical error (estimation) minimization problem in \ref{eq:min} as illustrated in \cite{daneshmand2016starting}.
Thus, solving the ERM with a statistical accuracy of $\delta_s$ in \ref{eq:rho} equal to the $V_ n$ is sufficient to provide a uniformly stable result, namely hypothesis $h$. This solution is denoted in literature by calculating the ERM within its statistical accuracy.

\subsection{Adaptive Sample Size}

Following the work of \cite{mokhtari2017first}, an adaptive sample size scheme is employed to take advantage of the nature of ERM, namely, the finite sum of functions drawn identically and independently from the same distribution to achieve a higher convergence rate with lesser computational complexities. However, the research in \cite{mokhtari2017first} only focuses into linearly convergent algorithms (strongly convex loss function are implemented with the aid of L-2 norm regularizer). The adaptive sample size scheme starts with a small portion of the training samples and solves the correspoding ERM within its statistical accuracy, then expands to include new samples with the original one and solves the ERM with the initial solution found by the previous sample and repeats until all samples are finished.
\\
\par
In other words given training data samples $\mathcal{Z}$ with $|\mathcal{Z}|=s$, we initialize the training with small sample $S_m \subset \mathcal{Z}$ and solve ERM in \ref{eq:ERM} within its statistical accuracy namely $\mathcal{O}(1/m^\alpha)$ to find $\hat{h}_m $ defined by some weights $w_m$. Then expand the training sample to include new samples such that $S_m \subset S_n \subset \mathcal{Z}$ and solve the ERM with initial solution of $w_m$ to find the ERM solution $w_n$.  Repeat this process until all data in ${s}$ are included .
\\
\par 
The relationship between the consecutive solutions $w_n$ and $w_m$ with ${n = 2m}$ (the increase is discussed in section 4) is established by the following theorem. The bound is expressed in terms of the first sample statistical solution to indicate that solving the first ERM problem with zero accuracy is not required.

\begin{theorem}
\label{th:th1}
Given the solution $w_m$ that solves the ERM with tolerance $\delta_m$ on sample $S_m \subset \mathcal{Z}$ such that in expectation $\mathbb{E}[L_m(w_m) - L_m(w^*_m)] \leq \delta_m$. Assume the there exist an optimal solution $w_n$ on sample $S_n$ such that $S_m \subset \mathcal{Z}$ and its statistical accuracy $V_n$, then in expectation we have the empirical risk difference is bouneded in expctation between the $w_n^*$ and $w_m$ as:
\begin{equation}
\label{eq:th1}
    \mathbb{E}[L_n(w_m) - L_n(w_n^*] \leq \delta_m + \frac{n-m}{n} (2V_{n-m} + V_m +V_n)
\end{equation}
\end{theorem}

\begin{proof}
 Starting by rewrite difference between the empirical losses using two models and denote $S_{n-m}$ the set of samples in $S_n \cap (S_m \cap S_n)^c$ thus:
\begin{align} 
\label{eq:pr1}
    \mathbb{E}[L_n(w_m) - L_n(w_n^*)] &= \mathbb{E}[\underbrace{(L_n(w_m) - L_m(w_m))}_{1}  + \underbrace{(L_m(w_m) - L_m(w_m^*))}_{2} \\
    & + \underbrace{(L_m(w_m^*) - L_m(w_n^*))}_{3} + \underbrace{(L_m(w_n^*) - L_n(w_n^*))}_{4}]
\end{align}
The first difference is bounded by Lemma 5 in \cite{mokhtari2017first} as \[\mathbb{E}[|L_n(w_m) - L_m(w_m)|] \leq \frac{n-m}{n} (V_{n-m} +V_m) \]
The second difference us the optimization error which is assumed to be : \[ \mathbb{E}[L_m(w_m) - L_m(w_m^*)] \leq \delta_m \]
The third difference is bounde above by zero since $w_m^*$ is the minimizer of the empire risk $L_m$. \[\mathbb{E}[L_m(w_m^*) - L_m(w_n^*)] \leq 0 \]
The forth difference is bounded by Lemma 5 in \cite{mokhtari2017first} as \[\mathbb{E}[|L_m(w_n^*) - L_n(w_n^*)|] \leq \frac{n-m}{n} (V_{n-m} +V_n) \]
Putting all four bounds back in \ref{eq:pr1} to obtain the result in theorem \ref{eq:1}.
\end{proof}
Theorem 1 asserts that even with the most accurate $L_m$ ERM solution i.e. $\  delta_m =0$, the subsequent problem $L_n$ with $w_0 = w_m$ will always have an optimal solution that has a dependency on the $V_m$. Thus solving the $L_m$ should be only withing $\mathcal{O}(V_m)$ only to reduce the computational complexity. The results \ref{eq:th1} in theorem \ref{th:th1} can be simplified if we consider $V_n = 1/n^\alpha$ and $n = 2m$ in Lemma \ref{lemma:lemma1}.
\begin{lemma}
\label{lemma:lemma1}
\begin{align}
\mathbb{E}[L_n(w_m) - L_n(w_n^*] &\leq \delta_m + \frac{1}{2} \left(\frac{2}{(n-m)^\alpha} + \frac{1}{m^\alpha} +\frac{1}{n^\alpha}\right)\\
& = \delta_m + \frac{1}{2} \left(\frac{2}{m^\alpha} + \frac{1}{m^\alpha} +\frac{1}{(2m)^\alpha}\right)\\
& = \delta_m + \frac{1}{2 } \left(3 +\frac{1}{2^\alpha}\right)V_m
\end{align}
\end{lemma}
\subsection{Computational Complexity}
The computational complexity of an algorithm is a measure of the algorithm's recruitment of computer resources, and the less computing required to accomplish one iteration in an iterative process, the simpler the algorithm is. Typically, it's measured in cost units associated with the algorithm; for example, some algorithms are assessed in gradient evaluation or number of iterations, while others require counting the total number of computing activities performed by the machine. First the smooth loss function assumption is stated as below.

Theorem \ref{th:th2} provides the minimal number of iterations required to solve the ERM on subset $S_n$ within statistical accuracy, i.e. $\mathbb{E}[L_n(w_n)-L_n(w_n^*)] \leq V_n$ given that the optimization algorithm has a sublinear convergence rate.
\begin{theorem}
\label{th:th2}
Given the initial solution $w_m$ and assuming the optimal solution of the ERM on the subset $S_n \subset \mathcal{Z}$ to be $w^*_n$, the sublinear convergence optimization algorithm needs the following iteration to solve the ERM within its statistical accuracy:
\begin{equation}
        T\geq \left[2^\alpha \left(\frac{5}{2 } +\frac{1}{2^{\alpha+1}}\right) \right]^\frac{1}{\zeta}
\end{equation}
Where T is the iteration number, and $\zeta$ is a positive constant determined by the optimization setting.
\end{theorem}
\begin{proof}
 Given the initial solution $w_m$ and assuming the optimal solution of the ERM on the subset $S_n \subset \mathcal{Z}$ to be $w^*_n$, the sublinear convergence optimization bounds the difference in expectation as:
 \begin{align}
     \mathbb{E}[L_n(w_T) - L_n(W_n^*]& \leq \frac{1}{T^\zeta} (L_n(w_m) - L_n(w^*_n))\\
     &\underset{a}{\leq}\frac{1}{T^\zeta}\left[V_m + \frac{1}{2 } \left(3 +\frac{1}{2^\alpha}\right)V_m\right]
 \end{align}
 Where the inequality (a) comes from the results in lemma \ref{lemma:lemma1} with $\delta_m = V_m$. In order to solve the ERM in $V_n$ accuracy we need to bound the RHS last equation by $V_n$ as follow:
 \begin{align}
     \frac{1}{T^\zeta}\left[V_m + \frac{1}{2 } \left(3 +\frac{1}{2^\alpha}\right)V_m\right] &\leq V_n\\
        \frac{1}{T^\zeta} \left(\frac{5}{2 } +\frac{1}{2^{\alpha+1}}\right) &\leq \frac{1}{2^\alpha}\\
        T\geq \left[2^\alpha \left(\frac{5}{2 } +\frac{1}{2^{\alpha+1}}\right) \right]^\frac{1}{\zeta}
 \end{align}
\end{proof}
The number of iterations T in theorem \ref{th:th2} ensures that the solution of any phase (stage) meets the statistical accuracy utilizing this lower constraint based on the iterative optimization algorithm being used. With a batch or sample of data, the ERM problem in equation \ref{eq:ERM} is solved until the statistical accuracy of that batch is guaranteed, and the solution is then employed as an initial solution for the next batch. The requirement of statistical accuracy, on the other hand, necessitates access to the unknown minimizer $w^*_n$, thus  Theorem \ref{th:th2} examines the minimum iterations needed such that an iterative method might utilize as a stopping criteria. Now the algorithm of solving the ERM problem in an adaptive way is illustrated in Algorithm 1.

\begin{algorithm}[!h]
\caption{Adaptive Sample Size iterative ERM solver Algorithm}\label{alg:cap}
\begin{algorithmic}
\label{alg:alg1}
\Require initial Sample size: $m_0$, Initial Solution $w^0$ such that $\mathbb{E}[L_{m_0}(w^0)-L_{m_0}(w^*)
\leq V_{m_0}]$, $\alpha$, $\zeta$ 
\Ensure $w_s \leq w_s^* +V_s$
\State $m \gets m_0$
\While{$n \leq |\mathcal{Z}|$}
\State $w_m \gets w^0$
\State $n \gets \min(2m,|\mathcal{Z}|)$
\State $T_{max} \gets \left[2^\alpha \left(\frac{5}{2 } +\frac{1}{2^{\alpha+1}}\right) \right]^\frac{1}{\zeta}\log(n)$
    \State Solve ERM problem on $S_n\subset\mathcal{Z}$ with $T_{max}$ and initial solution $w_m$
    \State $w^0 = w_n$
\EndWhile
\end{algorithmic}
\end{algorithm}

\section{Experiment}
The experiments in this part are carried out with first-order optimization algorithms that have a sub-linear convergence rate on an objective function that meets assumption 1, namely L-smooth and convex. The experiment's purpose is to assess the sub-optimality of these algorithms when adaptive sample size can be used against their fixed sample size counterpart. The gradient descent algorithm is selected from deterministic algorithms, whereas the ADAM algorithm is selected from stochastic algorithms. The logistic function with binary classification is the objective loss function to be minimized.\\
\par 
We Refer to Gradient Descent with adaptive sample size as adaptive gradient descent (adaGD) and for ADAM with adaptive sample size as adaptive ADAM (adaADAM).
The characteristics and total samples attributes of the datasets considered are detailed in table \ref{tab:table1}. Only the digits zero and eight are represented in binary in the MNIST databases. The Gradient Descent algorithm is first ran on every data set for a large number of iteration to obtain the optimal value. Then the ADAM algorithm have fixed parameters as : 1st-order exponential decay $\beta_1 = 0.9$, 2nd-order exponential decay $\beta_2=0.999$, step size $\eta=0.01$ and a small value $\epsilon=1e-8$ to prevent zero-division. The batch size is chosen to be 5 in ADAM and in adaADAM. The gradient descent step size is chosen based on the L-smoothness parameter values: $\gamma =1/L$.\\
\par 
\begin{table}[h]

    \centering
        \caption{Dataset Characteristics}
    \begin{tabular}{c c c c}
    \hline

     Name: & MNIST & RCV1 & a1a \\
     \hline
    training size & 6000 & 20242 & 1,605 \\
    testing size & 5774 & 677,399 & 30,956\\
    features size & 784 & 123 & 47,236\\
    \hline
    \end{tabular}
    \label{tab:table1}

\end{table}

\begin{figure}[!h]
    \centering
    \includegraphics[width=\textwidth]{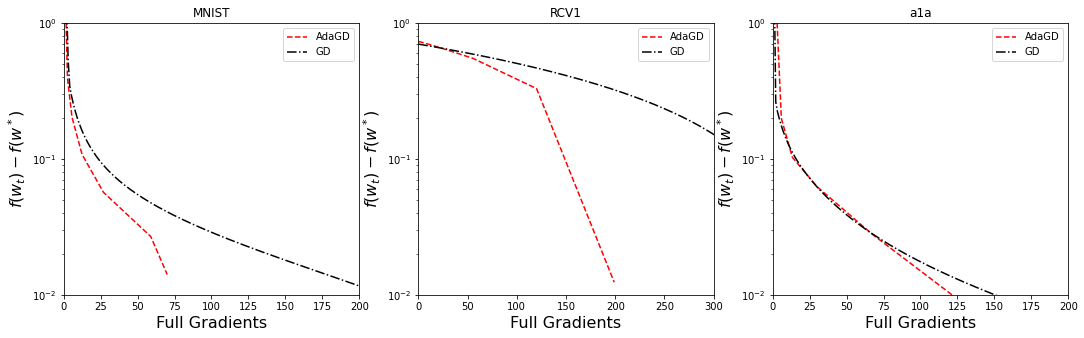}
    \caption{Comparison between Gradient descent and Adaptive sampling of Gradient descent}
    \label{fig:fig1}
\end{figure}
\begin{figure}[!h]
    \centering
    \includegraphics[width=\textwidth]{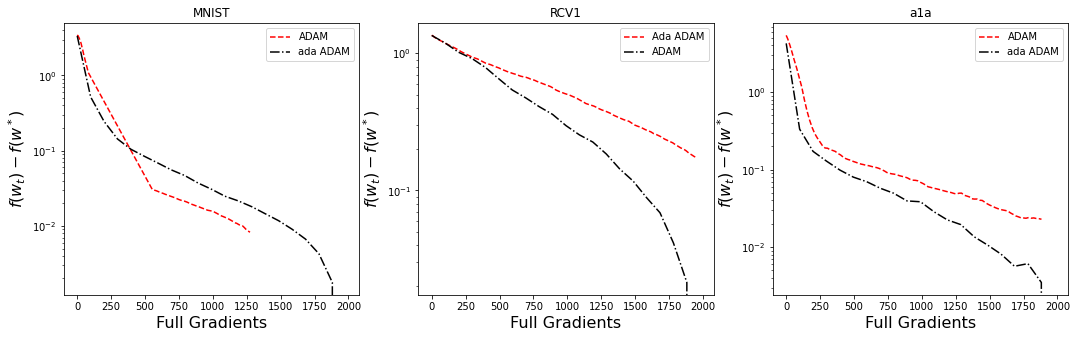}
    \caption{Comparison between ADAM  and Adaptive sampling of ADAM}
    \label{fig:fig2}
\end{figure}
\section{Discussion}

This paper presents an adaptive sampling technique to simplify the ERM problem for first-order optimization algorithms with sublinear convergence under terms of convexity and L-smoothness. Based on the carried experiments, we can infer that adaptive sampling generally resulted in faster convergence for sublinear problems. The adaptive Gradient (adaGD) has reduced the computational complexity of minimizing the logistic loss on the three datasets MNIST, RCV1, and a1a, as shown in figure \ref{fig:fig1}. However, the MNIST dataset has the greatest reduction in complexity, which is characterized in gradient evaluations, while the a1a dataset has the least.
\\
\par 
The reduction in computational complexity in adaptive ADAM (adaADAM) is not significant in some datasets, such as MNIST and RCV1, but it has proven to be significant in a1a as shown in figure \ref{fig:fig2}. The explanation for this could be that the ADAM algorithm stochastically shuffles the dataset after each epoch, which could result in the samples being repeated in different batches, which would employ the adaptive sampling technique implicitly. Future work related to exploring dataset characteristics that would limit convergence rate enhancement for sublinear problems through adaptive sampling is of interest.

\section*{Acknowledgments}
The work in this paper was a continuation of \cite{mokhtari2017first} and it was supported by the course instructor Dr. Bin Gu. 
\bibliographystyle{apalike}
\bibliography{refs}

\end{document}